\documentclass{article}
\usepackage[utf8]{inputenc}
\usepackage[round,authoryear]{natbib}
\usepackage[margin=1in]{geometry}
\usepackage{amsmath}
\usepackage{amssymb}
\usepackage{amsthm}
\usepackage{mathrsfs}
\usepackage[colorlinks=true,urlcolor=blue,linkcolor=red,citecolor=blue]{hyperref}%
\usepackage{mathtools}
\usepackage{graphicx}
\usepackage{subcaption}
\usepackage{array}
\usepackage{pifont}
\usepackage{algorithm}
\usepackage{algpseudocode}
\def\Real{\mathop{\mathbb{R}}\nolimits}

\newcommand{\bv}{\boldsymbol{v}}
\newcommand{\bw}{\boldsymbol{w}}
\newcommand{\bx}{\boldsymbol{x}}
\newcommand{\by}{\boldsymbol{y}}
\newcommand{\bz}{\boldsymbol{z}}
\newcommand{\bmu}{\boldsymbol{\mu}}
\newcommand{\bA}{\boldsymbol{A}}

\newcommand{\bX}{\boldsymbol{X}}

\newcommand{\bTheta}{\boldsymbol{\Theta}}

\newcommand{\bSigma}{\boldsymbol{\Sigma}}

\newcommand{\X}{\mathcal{X}}

\usepackage{algorithm}
\usepackage{algpseudocode}
\newtheorem{thm}{Theorem}

\newtheorem{cor}{Corollary}
\usepackage[blocks]{authblk}
\usepackage[T1]{fontenc}
\usepackage{babel}
\title{Automated Clustering of High-dimensional Data with a Feature Weighted Mean Shift Algorithm\thanks{To appear at the 35-th AAAI Conference on Artificial Intelligence, 2021.}}
\author[1]{Saptarshi Chakraborty\thanks{Joint first authors contributed equally to this work.}}
\author[2]{Debolina Paul$^\dagger$}
 \author[3]{Swagatam Das\thanks{Correspondence to: \href{mailto:swagatam.das@isical.ac.in}{swagatam.das@isical.ac.in}}}
 \affil[1]{Department of Statistics, University of California, Berkeley}
 \affil[2]{Indian Statistical Institute, Kolkata, India}
 \affil[3]{Electronics and Communication Sciences Unit, Indian Statistical Institute, Kolkata, India}

\date{\vspace{-5ex}}
\begin{document}

\maketitle

\begin{abstract}
Mean shift is a simple interactive procedure that gradually shifts data points towards the mode which denotes the highest density of data points in the region. Mean shift algorithms have been effectively used for data denoising, mode seeking, and finding the number of clusters in a dataset in an automated fashion. However, the merits of mean shift quickly fade away as the data dimensions increase and only a handful of features contain useful information about the cluster structure of the data. We propose a simple yet elegant feature-weighted variant of mean shift to efficiently learn the feature importance and thus, extending the merits of mean shift to high-dimensional data. The resulting algorithm not only outperforms the conventional mean shift clustering procedure but also preserves its computational simplicity. In addition, the proposed method comes with rigorous theoretical convergence guarantees and a convergence rate of at least a cubic order. The efficacy of our proposal is thoroughly assessed through experimental comparison against baseline and state-of-the-art clustering methods on synthetic as well as real-world datasets. 
\end{abstract}
\section{Introduction}
Clustering, a cornerstone of unsupervised learning, refers to the task of partitioning a dataset into more than one exhaustive and mutually exclusive groups, based on some measure of similarity \citep{xu2015comprehensive}. Some popular paradigms in clustering include center-based approaches such as $k$-means and its variants \citep{JAIN2010651}, hierarchical clustering \citep{10.5555/1756006.1859898}, spectral clustering \citep{ng2002spectral,hess2019spectacl}, density-based methods \citep{ester1996density}, convex clustering \citep{chi2015splitting}, kernel clustering \citep{dhillon2004kernel}, optimal transport based methods \citep{mi2018variational,chakraborty2020hierarchical}, model-based frequentist approaches \citep{mcnicholas2016model} and Bayesian methods \citep{archambeau2007robust,DBLP:conf/icml/KulisJ12}.

Most of the aforementioned algorithms inherently use the number of clusters ($k$) as an input. However, for real-world data, $k$ may not be known beforehand. Determining $k$ from the dataset itself has long been an open problem and has attracted a lot of attention from the relevant research community \citep{tibshirani2001estimating,hamerly2004learning,fischer2011number,DBLP:conf/icml/KulisJ12, chakraborty2018simultaneous, GUPTA201872, paulbayesian}.

Moreover, algorithms for solving $k$-means type non-convex clustering problems are prone to get stuck at local minima \citep{xu2019power}. Recent attempts to mitigate this issue approach the problem via annealing with a class of functions approximating the $k$-means type objective \citep{xu2019power,chakraborty2020entropy,paul2020kernel} or by taking a convex relaxation of the problem \citep{chi2015splitting,pelckmans2005convex,wang2018sparse}. There are also density-based algorithms like DBSCAN \citep{ester1996density, jiang2017density} or other mode seeking approaches like quick shift \citep{vedaldi2008quick, jiang2017density} which attempts to speed up mean shift. However, these methods either require $k$ as an input or their performance degrade in a high-dimensional setting, where the signal-to-noise-ratio is quite low.  

To find the number of clusters automatically and to learn various properties of the feature space, researchers have resorted to the mean shift (MS) paradigm \citep{cheng1995mean,su2017solving}. Mean shift has previously been used for mode seeking, object tracking, and automated clustering in the feature space.  

Suppose $\X=\{\bx_1,\dots,\bx_n\} \subset \Real^p$ be $n$ data points to be clustered. The mean shift initiates $n$ points $\by_1^{(0)},\dots,\by_n^{(0)}$ and updates $\by_i$ according to the following update rule:
\begin{equation}
\label{eq1}
\by_i^{(t+1)}= \frac{\sum_{j=1}^n K(\|\by_i^{(t)}-\bx_j\|/h) \bx_j}{\sum_{j=1}^n K(\|\by_i^{(t)}-\bx_j\|/h)},    
\end{equation}
until convergence. Here $K(\cdot)$ is a kernel function (e.g. the Gaussian kernel, $K(x)=\exp\{-x^2\}$) and $h$ is the bandwidth parameter. This version of the mean shift is often used to detect the mode(s) of the estimated density. Often, instead of updating based on the original data points $\bx_i$'s, one uses $\by_i^{(t)}$'s, the points obtained in the $t$-th step. In this case, the updates take the following form:
\begin{equation}
    \label{eq2}
    \by_i^{(t+1)}= \frac{\sum_{j=1}^n K(\|\by_i^{(t)}-\by^{(t)}_j\|/h) \by^{(t)}_j}{\sum_{j=1}^n K(\|\by_i^{(t)}-\by^{(t)}_j\|/h)}.
\end{equation}
As before, $y_i^{(0)}$ is initiated at $\bx_i$. This version of mean shift is referred to as Blurring Mean Shift (BMS). The updates in equation \eqref{eq2} can be thought of as putting a low-pass filter on the data and thus ``blurring" out irregularities in the data. Apart from clustering, BMS has been used for data denoising and manifold learning \citep{5539845}.

Despite their simplicity, both the usual and blurring mean shift perform poorly for high-dimensional data. This is primarily because of the use of Euclidean distance, which becomes less informative as the number of feature increases due to the curse of dimensionality \citep{donoho2000high}. High-dimensional datasets often contain only a few relevant/ discriminating features, along with a huge number of irrelevant/noisy features, which severely affect the performance of mean shift and other popular clustering algorithms.  

There is a rich literature on clustering high-dimensional data including subspace clustering \citep{kriegel2009clustering,elhamifar2013sparse}, bi-clustering \citep{chi2017convex}, dimensionality reduction based approaches \citep{jin2016influential}  and data-depth based approaches \citep{sarkar2019perfect}. However, most of these methods are computationally expensive. Towards finding efficient feature representation of high-dimensional data while clustering, weighted $k$-means \citep{huang2005automated} and Sparse $k$-means \citep{witten2010framework} have become benchmark algorithms for learning effective feature representations of such data. However, these methods also require $k$ as input and thus, lose their appeal to the practitioner who may not have handled the data before and wants to find the number of clusters in an unsupervised manner as well.      

This paper aims to develop a blurring mean shift based algorithm, which can automatically find an efficient feature representation of the data as well as the number of clusters \textit{simultaneously}. Called the Weighted Blurring Mean Shift (WBMS), we extend the merits of blurring mean shift to high-dimensional data, while preserving its computational cost. To achieve this, we introduce a feature weight vector to learn the importance of each feature as the data is smoothed. The resulting iterations lead to an elegant and simple algorithm with closed-form updates. The weight update scheme follows the philosophy that features with higher within-cluster variance contribute less in finding the cluster structure of the data \citep{huang2005automated,witten2010framework,chakraborty2020entropy}. The main contributions of this paper can be summarized as follows:
\begin{itemize}
    \item In Section \ref{pf}, we introduce the Weighted Blurring Mean Shift (WBMS) formulation as an intuitive extension of mean shift to high-dimensional data clustering. This simple formulation is found to be effective in finding out the number of clusters and also filter out the unimportant features from the data \textit{simultaneously}.
    \item As shown in Section \ref{implicit}, the obtained feature weights can be interpreted as the outcome of an entropy regularization on the within-cluster sum of squares. 
    \item The WBMS algorithm comes with closed-form updates and its convergence guarantees are discussed in Section~\ref{convergence}. 
    \item We also analyze asymptotic convergence properties of the data cloud under the WBMS algorithm in Section \ref{rate}. We analytically show that the convergence of the spread of the data cloud is at least of cubic order. 
    \item Through detailed experimental analysis, we show the efficacy of our proposed algorithm against state-of-the-art clustering techniques on a suit of simulated and real-world data in Section \ref{exp}. Our experimental results indicate that WBMS is especially effective, compared to its competitors, in a high-dimensional setting despite the low signal-to-noise-ratio.
\end{itemize}

Before proceeding to the details of WBMS, let us now present a motivating example demonstrating its potential. 
\subsubsection{A Motivating Example}
\begin{figure*}[!ht]
    \centering
    \includegraphics[height=0.38\textwidth,width=0.9\textwidth]{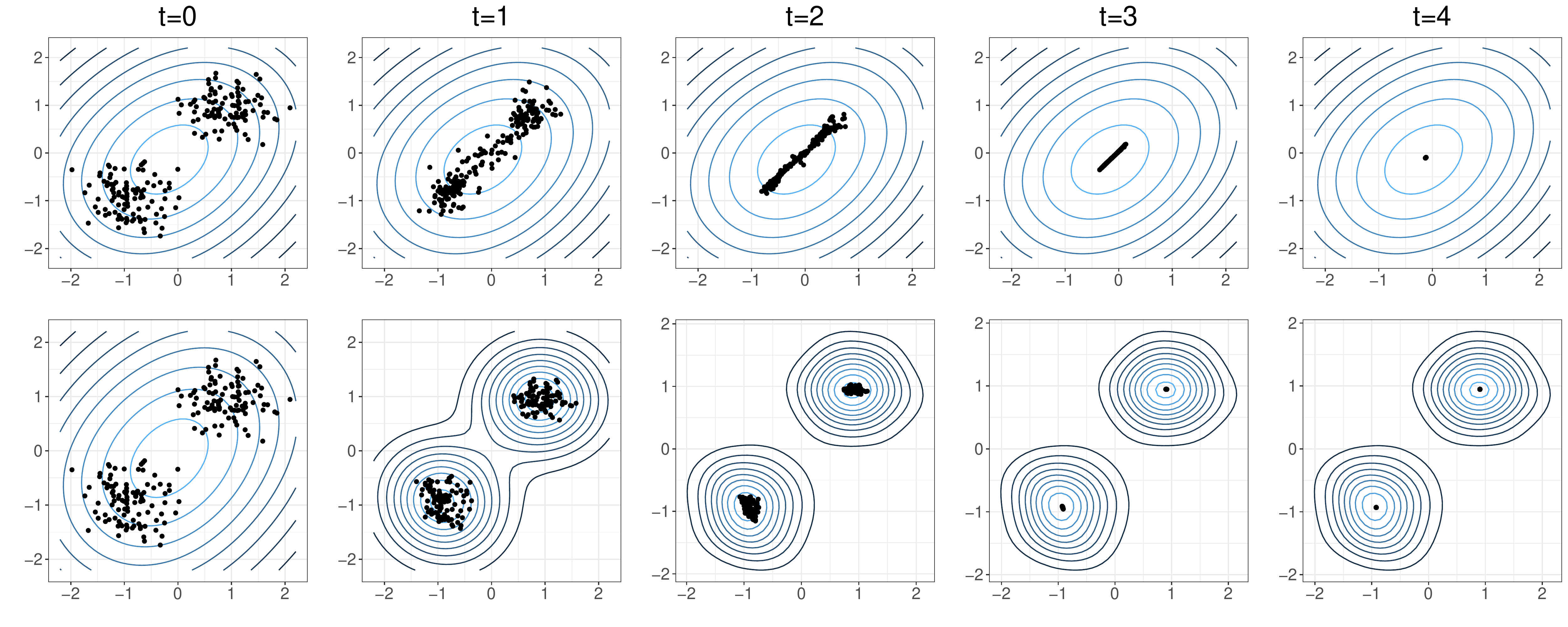}
    \caption{Performance of BMS (first row) and WBMS (second row) along with the contour plots of the estimated kernel density for the motivating example as the number of iterations ($t$) is increased. WBMS correctly identifies the two true cluster centroids and efficiently selects the relevant features, while the BMS fails to do so.}
    \label{fig:motivate}
\end{figure*}
 We generate a $200 \times 32$ dimensional data, called \texttt{data1}, which consists of two clusters with $100$ points each. The cluster structure of \texttt{data1} is fully contained in the first two features and the rest of the $30$ features are independently generated from a standard normal distribution, that contain no clustering information. We standardize ($z$-transform) the data before use and set the bandwidth $h=0.1$. In Fig.~\ref{fig:motivate}, we show the position of the data cloud, plotted in the first two relevant feature dimensions, for both BMS and WBMS as the number of iterations ($t$) is increased.  It can be easily seen that since BMS is more influenced by the combination of $30$ Gaussian features, the data cloud gradually converges to the origin. On the other hand, the WBMS correctly identifies the two important and informative features. Thereby using this information, the data cloud converges to the two cluster centroids. The estimated density of the original data cloud $\mathcal{X}$, $\hat{f}(\bx) \propto \sum_{i=1}^n \exp\{-\|\bx-\bx_i\|^2_{\bw}/h\}$, with $\bw$ being the feature weights found out by WBMS (see Section \ref{pf} for the definition of $\|\cdot\|_{\bw}$), also closely resembles the density of the data without the presence of the $30$ non-informative features.
\section{Weighted Blurring Means Shift}
\label{pf}
In this section, we formulate the Weighted Blurring Mean Shift (WBMS) algorithm and discuss some of its intriguing properties. Throughout this paper, $\mathcal{N}_p(\bmu,\Sigma)$ denotes the $p$-variate normal distribution with mean $\bmu$ and dispersion matrix $\Sigma$. $Unif(A)$ denotes the uniform distribution over the set $A$.
\subsection{Motivation}
Let $\bx_1,\dots,\bx_n \in \Real^p$ be $n$ data points to be clustered. We note that the blurring mean shift updates (equation \eqref{eq2}) uses the Euclidean distance $\|\bx-\by\|_2=\sqrt{\sum_{l=1}^p (x_l-y_l)^2}$. The Euclidean distance puts equal weight on each of the components $(x_l-y_l)^2$ and thus, not suitable when there are many noisy features, which are irrelevant to the clustering of the data. Recent research \citep{witten2010framework,chakraborty2020entropy} has been focused in replacing the usual Euclidean distance with the weighted distance $ \|\bx-\by\|_{\bw}=\sqrt{\sum_{l=1}^p w_l (x_l-y_l)^2}.$
Here $w_l \ge 0$ denotes the feature weight of the $l$-th feature. It can be easily checked that $\|\cdot\|_{\bw}$ defines a norm on $\Real^p$. A large feature weight, $w_l$ on the $l$-th feature gives more importance to the difference $|x_l-y_l|$, the discrimination between $\bx$ and $\by$ along the $l$-th coordinate vector. $\bw$ is called the feature weight vector and is usually normalized, i.e. $\mathbf{1}^\top \bw =1$. The feature weights are typically learned from the data and are updated every passing iteration. Normally, the feature weights are taken as some decreasing function of the within cluster sum of squares for that feature. 
\subsection{Formulation}\label{formulation}
We will use a similar update rule as in BMS (equation \eqref{eq2}). However, instead of the usual Euclidean distance, we will use the weighted distance $\|\cdot\|_{\bw}$. The update rule for the data points is given by,
\begin{equation}
    \label{eq4}
    \by_i^{(t+1)}= \frac{\sum_{j=1}^n K(\|\by_i^{(t)}-\by^{(t)}_j\|_{\bw^{(t)}}/h) \by^{(t)}_j}{\sum_{j=1}^n K(\|\by_i^{(t)}-\by^{(t)}_j\|_{\bw^{(t)}}/h)}.
\end{equation}
The feature weights are updated as follows:
\begin{equation}
    \label{eq5}
    w^{(t)}_l=\frac{\exp\{-\frac{1}{n\lambda}\sum_{i=1}^n(x_{il}-y_{il}^{(t)})^2\}}{\sum_{l^\prime=1}^p\exp\{-\frac{1}{n\lambda}\sum_{i=1}^n(x_{il^\prime}-y_{il^\prime}^{(t)})^2\}}.
\end{equation}
Algorithm \ref{alg1} gives a formal description of the weighted blurring mean shift algorithm.
\begin{algorithm}[h!]
\caption{Weighted Blurring Mean Shift (WBMS) Algorithm}\label{alg1}
\begin{algorithmic}
\State \textbf{Input}:  $\bx_1,\dots,\bx_n \in \Real^p$, $h,\lambda>0$
\State \textbf{Output}: $\by_1,\dots,\by_n$ and $\bw$.
\State Initialize $\by_i^{(0)} \leftarrow \bx_i$ for all $i=1,\dots,n$.
\State Initialize $w_l^{(0)}=\frac{1}{p}$, for all $l=1,\dots,p$.
\Repeat
 \State \textbf{Step 1}: Update $\by_i$'s by,
\[\by_i^{(t+1)} \leftarrow \frac{\sum_{j=1}^n K(\|\by_i^{(t)}-\by^{(t)}_j\|_{\bw^{(t)}}/h) \by^{(t)}_j}{\sum_{j=1}^n K(\|\by_i^{(t)}-\by^{(t)}_j\|_{\bw^{(t)}}/h)}.\]
 \State \textbf{Step 2}: Update $\bw$ by,
 \[w^{(t)}_l=\frac{\exp\{-\frac{1}{n\lambda}\sum_{i=1}^n(x_{il}-y_{il}^{(t)})^2\}}{\sum_{l^\prime=1}^p\exp\{-\frac{1}{n\lambda}\sum_{i=1}^n(x_{il^\prime}-y_{il^\prime}^{(t)})^2\}}.\]
 \Until{$\big|\max_{i,j}\|y_i^{(t+1)}-y_{j}^{(t+1)}\|_2 - \max_{i,j}\|y_i^{(t)}-y_{j}^{(t)}\|_2\big|$ converges}
\end{algorithmic}
\end{algorithm}
\subsection{Detecting the Clusters}
We observe that Algorithm \ref{alg1} outputs $\by_1,\dots,\by_n$, which are proxies for the cluster centroids. One should note that if $\bx_i$ and $\bx_j$ were originally in the same cluster, then $\by_i$ and $\by_j$ should be close to each other in the Euclidean sense, i.e. $\|\by_i-\by_j\|_2$ should be small enough. We construct an undirected graph with the adjacency matrix $\bA=((a_{ij}))$. For a prefixed tolerance $\epsilon$ (in our experiments, we take $\epsilon=10^{-5}$), we will take,
\[
a_{ij}=\begin{cases}
1, & \text{ if } \|\by_i-\by_j\|_2 < \epsilon\\
0, & \text{ Otherwise.}
\end{cases}
\]
Let $\mathcal{G}$ be a graph based on the adjacency matrix $\bA$ and vertices $\bx_1,\dots, \bx_n$. The clusters in $\mathcal{X}$ should ideally correspond to the connected components of $\mathcal{G}$. The number of clusters, $k$, corresponds to the number of connected components of $\mathcal{G}$. The connected components of $\mathcal{G}$ can easily be found by a Depth First Search or a Breadth First Search.  
\subsection{Implicit Entropy Regularization}\label{implicit}
We now show that the weight update scheme presented in equation \eqref{eq5} can be thought of as the outcome of an entropy regularization. At convergence $y_i^{(t)}$ can be treated as the cluster centroid corresponding to $\bx_i$. Thus, the within cluster sum of squares is given by,
\(\frac{1}{n}\sum_{i=1}^n \|\bx_i-\by_i^{(t)}\|_{\bw}^2\). Since we are imposing the constraint $\sum_{l=1}^pw_l=1$, minimization of this within cluster sum of squares will result in a trivial coordinate vector of $\Real^p$. As observed by \citep{chakraborty2020entropy}, this problem can be successfully avoided by adding an entropy incentive term as:
\begin{equation}\label{eqs1}
    \frac{1}{n}\sum_{i=1}^n \|\bx_i-\by_i^{(t)}\|_{\bw}^2 + \lambda \sum_{l=1}^p w_l \log w_l,
\end{equation}
where $\lambda>0$. Note that the second term in the above expression is the negative of Shannon's entropy of $\bw$. Minimizing \eqref{eqs1} subject to $\bw^\top \mathbf{1}=1$, results in the weight update formula, given in equation \eqref{eq5}. This is assured by the following theorem.
\begin{thm}
Let $\bw^\ast$ be the minimizer of \eqref{eqs1}, subject to $\bw^\top \mathbf{1}=1$. Then, 
\[w_l^\ast=\frac{\exp\{-\frac{1}{n\lambda}\sum_{i=1}^n(x_{il}-y_{il}^{(t)})^2\}}{\sum_{l^\prime=1}^p\exp\{-\frac{1}{n\lambda}\sum_{i=1}^n(x_{il^\prime}-y_{il^\prime}^{(t)})^2\}}.\]
\end{thm}
\begin{proof}
Problem $(5)$ can be rewritten as follows:
\[
\min_{\bw}\sum_{l=1}^p D_lw_l +\lambda \sum_{l=1}^p w_l \log w_l
\]
where $D_l=\frac{1}{n}\sum_{i=1}^n(x_{il}-y_{il}^{(t)})^2$. 

This problem can now be solved using the constraints $\bw^\top \mathbf{1}=1$, $w_l\geq 0 \text{ } \forall l=1, \dots, p$. The Lagrangian is thus given by,
\[
L= \sum_{l=1}^p D_lw_l +\lambda \sum_{l=1}^p w_l \log w_l - \alpha \bigg(\sum_{l=1}^pw_l-1\bigg)
\]
Putting $\frac{\partial L}{\partial w_l}=0$, we have,
$$D_l+\lambda(1+\log w_l)-\alpha=0$$
Thus, we have $w_l \propto \exp(-\frac{D_l}{\lambda})$. 

Using the condition $\bw^\top \mathbf{1}=1$, we have, 
\begin{align*}
    w_l^\ast & = \frac{\exp(-\frac{D_l}{\lambda})}{\sum_{l'=1}^p\exp(-\frac{D_l'}{\lambda})}\\
    & = \frac{\exp\{-\frac{1}{n\lambda}\sum_{i=1}^n(x_{il}-y_{il}^{(t)})^2\}}{\sum_{l^\prime=1}^p\exp\{-\frac{1}{n\lambda}\sum_{i=1}^n(x_{il^\prime}-y_{il^\prime}^{(t)})^2\}}.
\end{align*}
This solution to the relaxed problem also satisfies the non-negativity condition. Hence the result.
\end{proof}
\section{Theoretical Properties and Convergence Guarantees}
\label{theo}
We will now discuss some of the interesting properties of WBMS. In particular, we prove that WBMS converges after a finite number of iterations for any fixed tolerance. We also find the asymptotic rate of convergence of the cluster variance as the number of points. Some related theoretical work in this area can be found in \citep{carreira2006fast,chen2015convergence,huang2018convergence,rocha2020towards}.
\subsection{Convergence Guarantee}
\label{convergence}
In this section, we will discuss some of the theoretical properties of the WBMS algorithm. For any set $ A \subseteq \Real^p$, let $\mathcal{C}(A)$ denote the convex hull of $A$. We begin our analysis by proving that the convex hulls of $\{\by^{(t)}_1,\dots,\by_n^{(t)}\}$ for a decreasing sequence of sets in Theorem~\ref{thm2}. 
\begin{thm}\label{thm2}
Let $C_t = \mathcal{C} (\{\by^{(t)}_1,\dots,\by_n^{(t)}\})$. Then $\{C_t\}_{t=0}^\infty$ constitutes a decreasing sequence of sets, i.e. 
\[C_0 \supseteq C_1 \supseteq \dots C_t \supseteq C_{t+1} \supseteq \dots\]
\end{thm}
\begin{proof}
From equation \eqref{eq4}, we note that $\by_{i}^{(t+1)}$ can be represented as a convex combination of $\by_1^{(t)},\dots,\by_n^{(t)}$. This implies that $\by_i^{(t+1)} \in C_t$, for all $i =1,\dots,n$. Thus $C_t$ is a convex set, which contains $\{\by_1^{(t+1)},\dots,\by_n^{(t+1)}\}$. Since $C_{t+1}$ is the smallest convex set containing $\{\by_1^{(t+1)},\dots,\by_n^{(t+1)}\}$, $C_{t} \supseteq C_{t+1}$.
\end{proof}
Since $\{C_t\}_{t=0}^\infty$ forms a decreasing sequence of sets, we immediately get that $C_t$ converges in the following corollary.
\begin{cor}\label{cor1}
$\lim_{t \to \infty} C_t$ exists and is given by,
\(\lim_{t \to \infty} C_t = \cap_{t=1}^\infty C_t.\)
\end{cor}
\begin{proof}
$\{C_t\}_{t=1}^\infty$ constitutes a decreasing sequence of sets. The result directly follows from applying monotone convergence theorem for sets \citep{rudin1964principles}.
\end{proof}
In Theorem~\ref{thm3}, we derive that for any fixed tolerance, the convergence criterion of Algorithm~\ref{alg1} is satisfied after a number of finite iterations.
\begin{thm}\label{thm3}
For any pre-fixed tolerance level $\delta$, there exists $T \in \mathbb{N}$ such that \[\big|\max_{i,j}\|y_i^{(t+1)}-y_{j}^{(t+1)}\|_2 - \max_{i,j}\|y_i^{(t)}-y_{j}^{(t)}\|_2\big| < \delta,\] for all $t \ge T$.
\end{thm}
\begin{proof}
From Corollary \ref{cor1}, $\{C_t\}_{t=1}^\infty$. This implies that the vertices of $C_t$ converges. Let  $a_t=\max_{i,j}\|y_i^{(t)}-y_{j}^{(t)}\|_2$. Since $a_t$ is a function of the vertices of $C_t$, $\{a_t\}_{t=1}^\infty$ also converges.  Since $\{a_t\}_{t=1}^\infty$ is a convergent sequence of reals, $\{a_t\}_{t=1}^\infty$ is Cauchy \citep{rudin1964principles}. Thus for any $\delta>0$, there exists $T \in \mathbb{N}$ such that \(\big|\max_{i,j}\|y_i^{(t+1)}-y_{j}^{(t+1)}\|_2 - \max_{i,j}\|y_i^{(t)}-y_{j}^{(t)}\|_2\big| < \delta,\) for all $t \ge T$.
\end{proof}
\subsection{Convergence Rate}
\label{rate}
Let us now discuss the behavior of a Gaussian cluster under WBMS. We will show that the Gaussian cluster shrinks towards its mean with at least a cubic convergence rate. Let $\phi(\bx;\bmu,\bSigma)$ denote the Gaussian probability density function with mean $\bmu$ and dispersion matrix $\bSigma$. In order to remove the dependency on the random process, we take an infinite sample, distributed in the whole of $\Real^p$ according to the density $q(\bx)$. For simplicity, we consider the Gaussian kernel. The kernel density estimate at $\bz$, based on the data $\by_1^{(t)},\dots,\by_n^{(t)}$ is given by,
\[\hat{p}_t(\bz) = \frac{c(h,\bw^{(t)})}{n} \sum_{j=1}^n \exp\big\{-\|\bz-\by_j^{(t)}\|^2_{\bw^{(t)}}/h\big\}.\]
Here $c(h,\bw^{(t)})$ is a constant depending only on $h$ and $\bw^{(t)}$. From equation \eqref{eq4}, we get, 
\begin{align*}
    \by_i^{(t+1)}  & = \frac{\sum_{j=1}^n \exp\{-\|\by_i^{(t)}-\by^{(t)}_j\|^2_{\bw^{(t)}}/h\} \by^{(t)}_j}{\sum_{j=1}^n \exp\{-\|\by_i^{(t)}-\by^{(t)}_j\|^2_{\bw^{(t)}}/h\}}\\
    & = \frac{c(h,\bw^{(t)})}{n} \sum_{j=1}^n \frac{\exp\{-\|\by_i^{(t)}-\by^{(t)}_j\|^2_{\bw^{(t)}}/h\} \by^{(t)}_j}{\frac{c(h,\bw^{(t)})}{n}\sum_{j=1}^n \exp\{-\|\by_i^{(t)}-\by^{(t)}_j\|^2_{\bw^{(t)}}/h\}}\\
    & = \frac{c(h,\bw^{(t)})}{n} \sum_{j=1}^n \frac{\exp\{-\|\by_i^{(t)}-\by^{(t)}_j\|^2_{\bw^{(t)}}/h\} \by^{(t)}_j}{\hat{p}_t(\by_i^{(t)})}\\
    & \approx c(h,\bw^{(t)})\int \by \frac{\exp\{-\|\by_i^{(t)}-\by\|^2_{\bw^{(t)}}/h\}}{p_t(\by_i^{(t)})} q_t(\by) d\by\\
    & = \int \by \frac{\phi(\by_i^{(t)}-\by; 0, \frac{h}{2} diag(\frac{1}{\bw^{(t)}}))}{p_t(\by_i^{(t)})} q_t(\by) d\by.
\end{align*}
Thus, if the number of samples is large, each data point $\bz$ is replaced by the conditional expectation 
\(E(\by|\bz) = \int \by p_t(\by|\bz) d\by,\)
where, $p_t(\by|\bz) = \phi(\bz-\by; 0, \frac{h}{2}diag(1/\bw^{(t)})) q_t(\by)/p_t(\bz)$. For simplicity of exposition, we begin with $\bx$, which follows a Gaussian distribution with mean $\mathbf{0}$ and dispersion matrix $\bSigma=diag(\sigma_1^2,\dots,\sigma^2_p)$. From the above analysis, it is clear that at a population level, the WBMS can be thought of as taking consecutive conditional expectations w.r.t $p_t(\by|\bx_t)$. Here $\bx_t$ denote the population at the $t$-th step of the algorithm with $\bx_0=\bx$. The following theorem asserts that $\bx_t$ is also normally distributed. 
\begin{thm}\label{thm4}
Let $\bx_0 \sim \mathcal{N}_p(\mathbf{0}, \Sigma)$ and $\bx_{t+1}=\int \by p_t(\by|\bx_t) d\by$. Here $p_t(\cdot)$ denotes the distribution of $\bx_t$ and $p_t(\by|\bz) = \phi(\bz-\by; 0, \frac{h}{2}diag(1/\bw^{(t)})) q_t(\by)/p_t(\bz)$. 
Then, $\bx_t \sim \mathcal{N}_p(\mathbf{0}, diag((s_1^{(t)})^2,\dots,(s_p^{(t)})^2))$, with 
\begin{equation}
\label{eq6}
s_l^{(t+1)}= (1+h(s^{(t)}_l)^2/2w_l^{(t)})^{-1}s_l^{(t)}.  
\end{equation}
\end{thm}
\begin{proof}We first find the distribution of $\bx_t$.
\begingroup
\allowdisplaybreaks
\begin{align*}
     p_1(\by|\bx_0)  & =\phi\left(\bx_0-\by; 0, \frac{h}{2}diag(1/\bw^{(0)})\right) q_0(\by)/p_t(\bx_0)\\
    & \propto \phi\left(\bx_0-\by; 0, \frac{h}{2}diag(1/\bw^{(0)})\right) q_0(\by)\\
    & \propto \exp\{- \frac{1}{h}\sum_{l=1}^p w_l^{(0)} (y_l-x_l^{(0)})^2\} \exp\{-\sum_{l=1}^p y_l^2/(2 \sigma_l^2)\}\\
    & \propto \exp\bigg\{- \frac{1}{2}\sum_{l=1}^p \frac{ \bigg(y_l-\frac{2 w_l^{(0)} x_l^{(0)}/h}{2w_l^{(0)}/h+1/\sigma_l^2}\bigg)^2}{ \frac{1}{(2w_l^{(0)}/h+1/\sigma_l^2)}}\bigg\}.
\end{align*}
\endgroup
Thus, 
$\bx_1 = E(\by|\bx_0) = \big(\frac{2 w_1^{(0)} x_1^{(0)}/h}{2w_1^{(0)}/h+1/\sigma_1^2},\dots, \frac{2 w_p^{(0)} x_p^{(0)}/h}{2w_p^{(0)}/h+1/\sigma_p^2}\big)$. Thus $\bx_1$ is also a Gaussian distribution, with mean $\mathbf{0}$ and dispersion matrix as
$\text{Var}(\bx_1)=diag((s_1^{(1)})^2,\dots,(s_p^{(1)})^2)$, with 
\[s_l^{(1)}=\frac{2 w_l^{(0)} s_0 /h}{2w_l^{(0)}/h+1/\sigma_l^2} = \frac{1}{1+\frac{h}{2w_l^{(0)}}(s^{(0)}_l)^2}s_l^{(0)}.\]
Here $s^{(0)}_l=\sigma_l$. By an inductive argument, it is easy to see that, $\bx_{t+1}$ is also a Gaussian distribution, with mean $\mathbf{0}$ and dispersion matrix as
$\text{Var}(\bx_{t+1})=diag((s_1^{(t+1)})^2,\dots,(s_p^{(t+1)})^2)$, with
\(
    s_l^{(t+1)}= \big(1+\frac{h}{2w_l^{(t)}}(s^{(t)}_l)^2\big)^{-1}s_l^{(t)}.
\)
\end{proof}
Thus, the sequence of standard deviations $\{s_l^{(t)}\}$ form a decreasing sequence, which is bounded below. Hence, by monotone convergence theorem of real sequences \citep{rudin1964principles}, $\{s_l^{(t)}\}$ also converges. Let this limit be $s_l$. Hence $\{\bx^{t}_l\}$, the $l$-th coordinate of $\bx_t$, converges in distribution to either $\mathcal{N}(0, s_l^2)$ (if $s_l>0$) or the degenerate distribution at $0$ (if $s_l=0$). We also observe that at a population level,
\[w_l^{(t)} = \frac{\exp\{-E(x_l^{(0)}-x_l^{(t)})^2/\lambda\}}{\sum_{l^\prime=1}^p\exp\{-E(x_{l^\prime}^{(0)}-x_{l^\prime}^{(t)})^2/\lambda\}}.\]
Since $\bw_{(t)}$ is a continuous function of $\bv_t$, where \(\bv_t=(E(x_1^{(0)}-x_1^{(t)})^2,\dots,E(x_p^{(0)}-x_p^{(t)})^2).\)
Since $\bx_t$ converges in distribution, $\bv_t$ also converges, which in turn implies that $\bw^{(t)}$ converges. Let the limit be $\bw$. The following theorem asserts that $s_l^{(t)} \to 0$, as $t \to \infty$, which in turn implies that $\bx_t$ converges to $\mathbf{0}$, in distribution.
\begin{thm}
Let $\bx_t$, $s_l^{(t)}$ be as in Theorem \ref{thm4}. Then, $s_l=\lim_{t \to \infty} s_l^{(t)}=0$ and the order of convergence of $\{s_l^{(t)}\}_{t=1}^\infty$ is at least cubic. Moreover, the asymptotic rate of convergence of $\{s_l^{(t)}\}$ is $\frac{2 w_l}{h}.$
\end{thm}
\begin{proof}
 We first consider the case $\lim_{t \to \infty} w^{(t)}_l = 0$. From equation \eqref{eq6}, we observe that $\lim_{t \to \infty} s_l^{(t+1)} =0$. Now if $\lim_{t \to \infty} w^{(t)}_l > 0$, we take limit as $t \to \infty$ on both sides of equation \eqref{eq6} and get,
\[s_l = \frac{1}{1+\frac{h}{2w_l}(s_l)^2} s_l \implies s_l = 0, \, \forall \, l=1,\dots,p.\]

We will now find the order of convergence of $s_l$ to zero. We say that the order of convergence is $m$ if $m$ is the largest integer such that 
\[r_s = \lim_{t \to \infty} \frac{|s_l^{(t+1)} - s_l|}{|s_l^{(t)} - s_l|^m} = \lim_{t \to \infty} \frac{(s_l^{(t)})^{3-m}}{(s_l^{(t)})^2+\frac{h }{2 w_l}} < \infty. \]
This occurs at least for the case $m=3$, i.e. the convergence rate of $\{s^{(t)}\}_{t=1}^\infty$ is at least cubic. If $m=3$, the asymptotic rate of convergence is given by, $r_s = \frac{2 w_l}{h}$.
\end{proof}
 Thus, the asymptotic convergence rate is smaller if $w_l$ is smaller, meaning that if a feature is deemed to have smaller feature weight, i.e. little relevance in containing the cluster structure, the convergence along that feature is faster.

Also note that since $\bx_t$ converges in distribution to $\mathbf{0}$, $w_l^{(t)}$, being a bounded continuous function of $\bx_t$ also converge in distribution to 
\[ w_l  =\frac{\exp\{-E(x_l^{(0)}-0)^2/\lambda\}}{\sum_{l^\prime=1}^p\exp\{-E(x_{l^\prime}^{(0)}-0)^2/\lambda\}}   \propto  \exp \{-\sigma_l^2/\lambda\}.\]
Thus, in limit, features with larger variances receive smaller feature weights compared to features with smaller variances.
\section{Experimental Results}
\label{exp}
In this section, we compare WBMS with  classical baselines and state-of-the-art automated clustering algorithms. To evaluate the performances, we use Normalized Mutual Information (NMI) \citep{vinh2010information} and Adjusted Rand Index (ARI)  \citep{hubert1985comparing} between the ground truth and the obtained partition. For both the indices, a value of $1$ indicates perfect clustering while a value of $0$ indicates completely random class labels. Since our algorithm is developed for automated clustering, apart from the baseline $k$-means \citep{macqueen1967some}, we compare our method with Blurring Mean Shift (BMS) \citep{5539845}, Gaussian means ($G$-means) \citep{hamerly2004learning}, Dirichlet Process means ($DP$-means) \citep{DBLP:conf/icml/KulisJ12}, Entropy Weighted $DP$-means (EWDP) \citep{paulbayesian}, and the Robust Continuous Clustering (RCC) \citep{shah2017robust} algorithm. For fair comparison against the other automated clustering methods, the number of clusters in $k$-means is supplied through the Gap statistics method \citep{tibshirani2001estimating}. It should be noted $WG$-means and RCC already entails higher computational complexity and we do not consider other alternative clustering techniques, which require $k$ as an input. Since $k$-means and $G$-means depend on the random seeding, we run each algorithm $20$ times independently on each data and the report the average performance. All the datasets are normalized ($z$-transform) before use. In our experiments, we use the Gaussian kernel. We observed that for the proposed algorithm, $h \in (0.1,1)$ and $\lambda \in [1,20]$ preserves a consistently good level of performance. All the other algorithms are tuned using their standard protocols. A pertinent ablation study is provided in the Appendix. Additional experiments and runtime comparisons also appear therein. For the sake of reproducibility, all the codes and datasets are available at \url{https://github.com/SaptarshiC98/WBMS}.
\subsection{Simulation Studies}
\label{sim}
\begin{figure*}
    \centering
    \includegraphics[height=0.24\textwidth,width=0.85\textwidth]{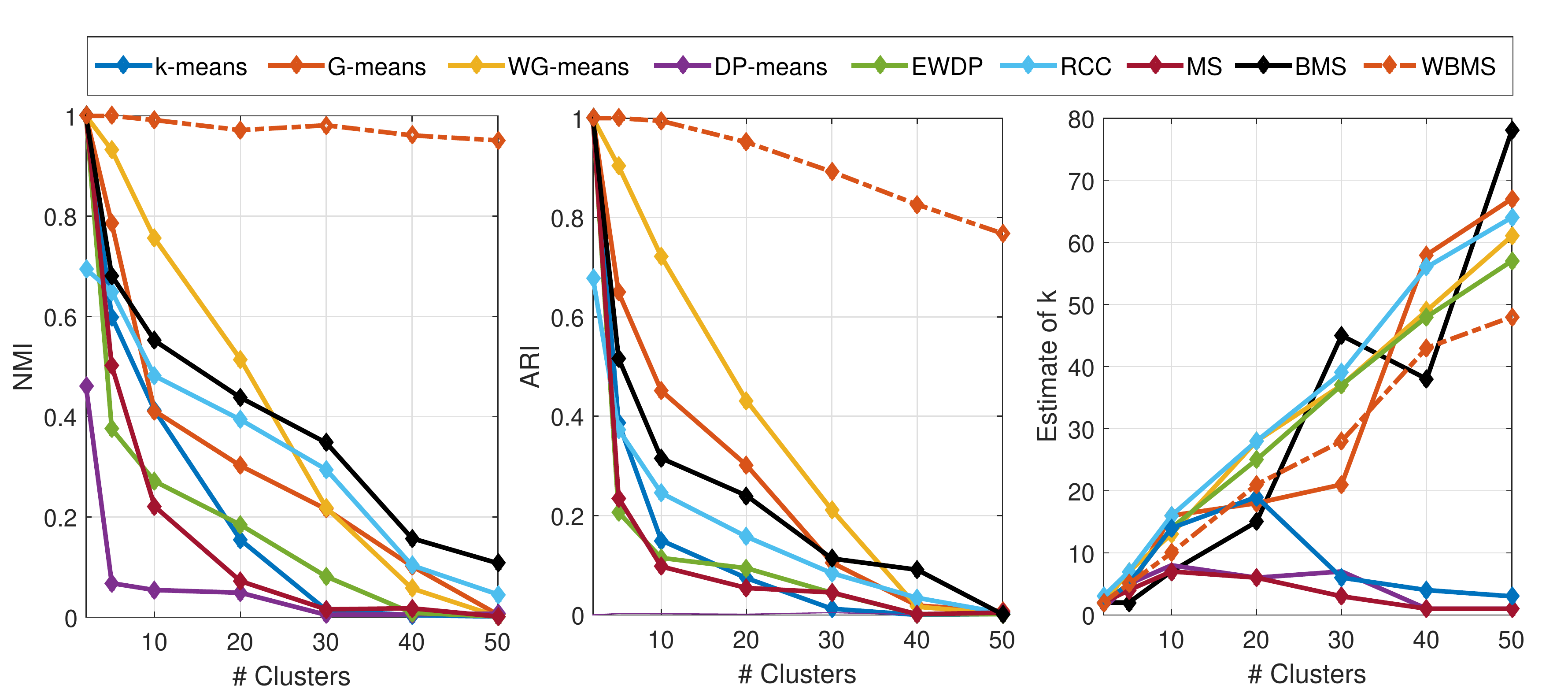}
    \caption{Performances of peer algorithms in terms of NMI (left), ARI (middle) and the estimated number of clusters (right) as the number of clusters increases in simulation 1. It is easily observed that WBMS consistently resembles the ground truth, while the other algorithms fail.}
    \label{cluster}
\end{figure*}
We now discuss the behavior of the peer algorithms through a set of simulation studies.
\subsubsection{Simulation 1: Effect of increasing $k$}
We now examine the behavior of the WBMS algorithm as the number of cluster $k$ increases. In this simulation study, we take $n=20 \times k$, $p=20$, while $k$ varies from $2$ to $50$. Let $\bTheta$ be the $k \times p$ real matrix, whose rows represent the $k$ cluster centroids. We generate $\bTheta$ as follows.  
\begin{itemize}
    \item Generate $\theta_{jl} \sim Unif(0,1)$ independently for $j=1,\dots,k$ and $l=1,\dots,5$.
    \item Set $\theta_{jl}=0$ for $j=1,\dots,k$ and $l=6,\dots,20$.
\end{itemize}
After generating $\bTheta$, we generate the $n \times p$ data matrix $\bX$ as:
\begingroup
\allowdisplaybreaks
\begin{align*}
c_i  \sim  Unif(\{1,&\dots,k\});   x_{il} \sim \mathcal{N}_1(\theta_{c_i, l},0.02^2) \text{ if } l \in \{1,\dots,5\}\\
    x_{il} & \sim \mathcal{N}_1(0,1) \text{ if } l \in \{6,\dots,20\}.
\end{align*}
\endgroup
From the data generation procedure, it is easy to observe that only the first five features contain the cluster structure of the data. We run all the peer algorithms on each of the datasets. Fig.~\ref{cluster} shows the average NMI and ARI values between the ground truth and the obtained partition. We also plot the average estimated number of clusters ($\hat{k}$) against the true number of clusters ($k$). It can be easily observed that WBMS consistently outperforms the other peer algorithms not only in terms of clustering performance but also in finding the true number of clusters.
\subsubsection{Simulation 2: Effect of increasing \textit{p}}
\begin{figure*}[ht]
    \centering
    \includegraphics[height=0.24\textwidth,width=0.85\textwidth]{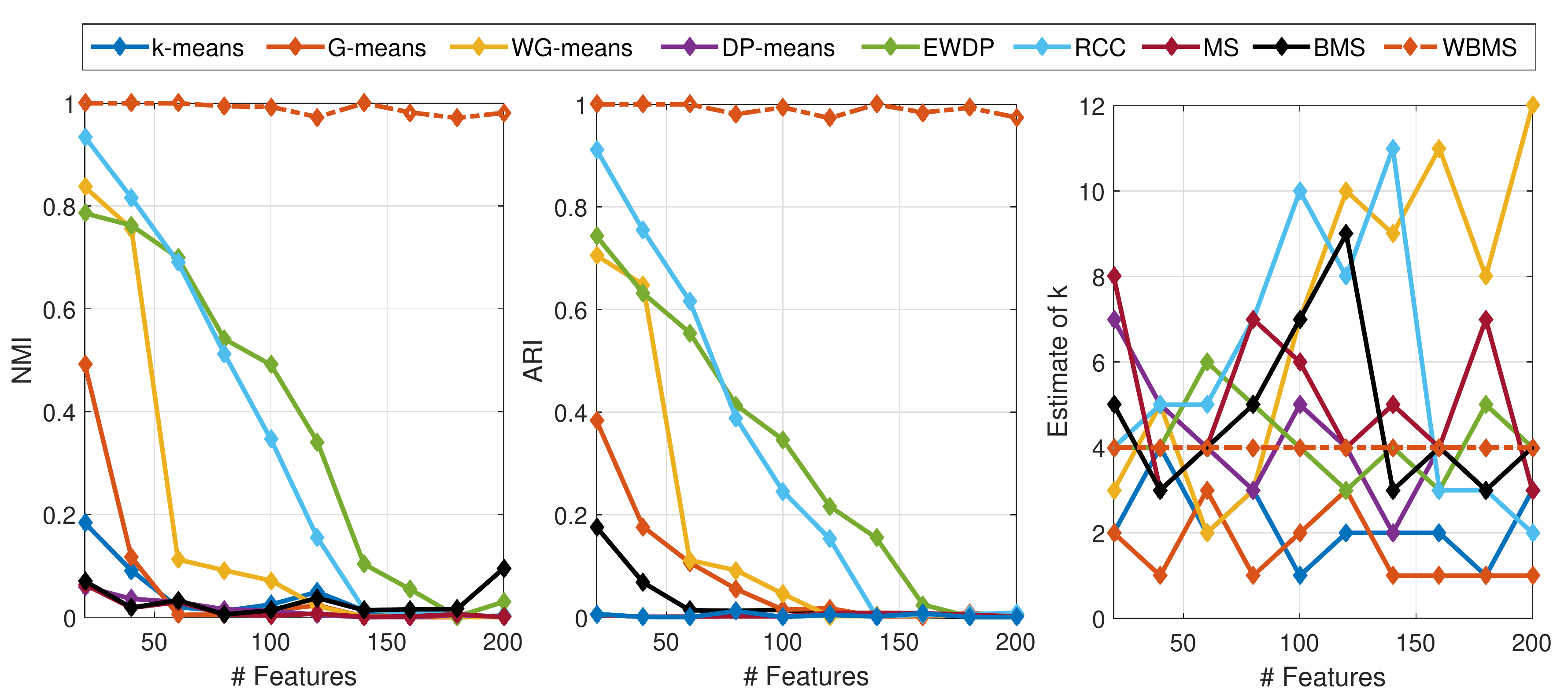}
    \caption{Performances of peer algorithms in terms of NMI (left), ARI (middle) and the estimated number of clusters (right) as the number of features increases in simulation 2. It is easily observed that WBMS consistently resembles the ground truth, while the other algorithms fail as the dimensionality of the data increases. }
    \label{feature}
\end{figure*}
\begin{figure*}[ht]
    \centering
    \includegraphics[height=0.24\textwidth,width=0.90\textwidth]{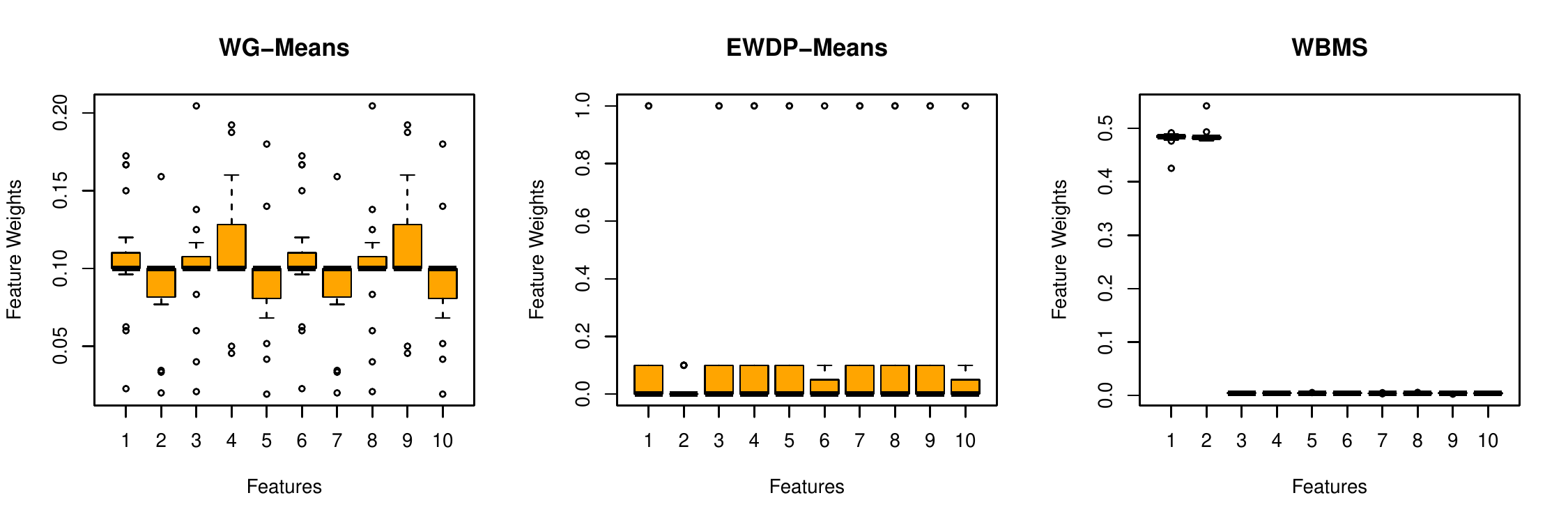}
    \caption{Boxplot shows that WBMS consistently identifies true features while WG-means and EWDP-means fail to do so.}
    \label{weight}
\end{figure*}
This experiment assesses the performance of WBMS as the number of features grows. We generate $n=100$ observations with $k=4$ clusters. We increase the dimension from $p=20$ to $p=200$ at the difference of $20$ to observe the effect of growing features. In each case, the number of informative features is fixed at $5\%$ of the total number of features with an exception for the case $p=20$, where we take $2$ informative features. The clusters are spherical without any overlaps and are generated from Normal distributions with variance $0.3$ and means generated from $Unif(0,1)$ distribution. The non-informative features are generated from $\mathcal{N}_1(0,1)$. Fig.~\ref{feature} compares the NMI, ARI and estimated number of clusters for all the peer algorithms. From Fig.~\ref{feature}, we can easily observe that among all peer algorithms, not only the WBMS algorithm performs best in terms of NMI and ARI values, it also estimates the number of clusters perfectly.

\subsubsection{Simulation 3: Feature Selection}
\begin{figure*}[!ht]
    \centering
    \includegraphics[height=0.3\textwidth,width=0.9\textwidth]{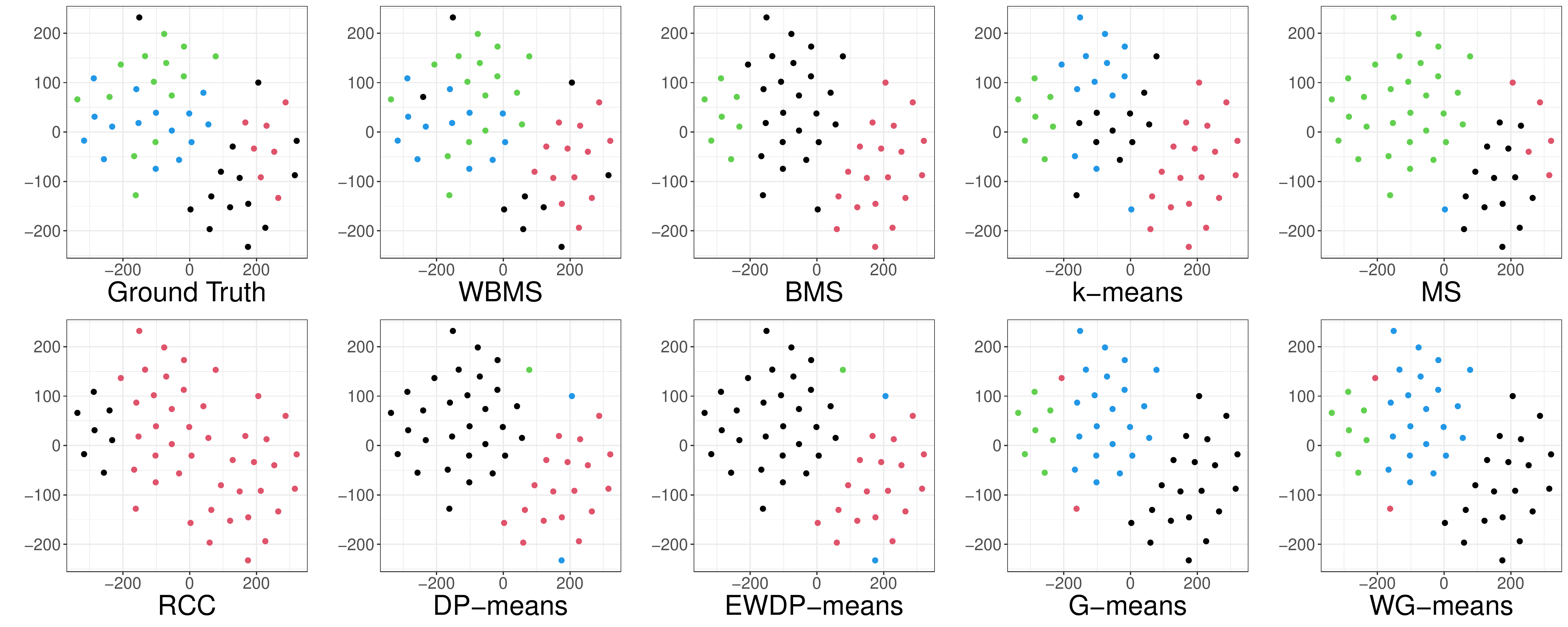}
    \caption{\texttt{t-SNE} plots for the GLIOMA dataset, color-coded by the partitions obtained by each peer algorithm.}
    \label{glioma}
\end{figure*}

We now examine the feature weighting properties of WBMS. Here, we take $n=100$, $p=10$, $k=4$ and follow the data generation procedure described in Simulation $2$. Without loss of generality, we take the first two features to be informative that actually contain the cluster structure. We compare the feature weights obtained by WBMS against those obtained by $WG$-means and EWDP. We record the feature weights obtained by these two algorithms along with proposed WBMS over 100 replicates of the data. The box plots for these $100$ optimal feature weights are shown in Fig. \ref{weight} for all the three algorithms. The proposed WBMS successfully assigns almost all weights to informative features $1$ and $2$, even in this low signal-to-noise-ratio situation. Meanwhile, its peers $WG$-means and EWDP fail to do so. 

\subsection{Case Study on Glioma Data}
We now evaluate the performance of our algorithm with a specific case study on the microarray dataset, Glioma \citep{nutt2003gene}. The dataset comprises of $4434$ gene expression levels, collected over $50$ samples. There are four natural classes in the data viz. Glioblastomas (CG), Non-cancer Glioblastomas (NG), Cancer Oligodendrogliomas (CO), and Non-cancer Oligodendrogliomas (NO). In our experimental study, we compare our proposed WBMS algorithm to all the algorithms described in Section \ref{sim}. To visualize the clustering results, we perform a \texttt{t-SNE} \citep{maaten2008visualizing} and show the resulting embedding in Fig.~\ref{glioma}, color-coded with the partitions obtained from the peer algorithms. WBMS closely represents the ground truth compared to its competitors. This is also easily seen from Table~\ref{real}.

\subsection{Performance on Real Data Benchmarks}
\begin{table}[h]
    \centering
    \begin{tabular}{l c c c}
    \hline
    Datasets & \# Datapoints & \# Features & \# Clusters \\
    \hline
      GLIOMA   & 50 & 4434 & 4  \\
       Appendicitis  & 107 & 7 & 2 \\
       Zoo & 101 & 16 & 7  \\
       Mammographic & 830 & 5 & 2\\
       Yale & 165 & 1024 &  15\\
       nci9 & 60 & 9712 & 9 \\
       Lymphoma & 62 & 4026 & 3\\
       Movement Libras & 360 & 90 & 15 \\
       GCM & 191 & 16063 & 15 \\
       \hline
    \end{tabular}
    \caption{The dimensions of various real data, used in our experiments along with the true number of clusters.}
    \label{data_dimension}
\end{table}
\begin{table*}[!ht]
    \centering
    \resizebox{\textwidth}{!}{
    \begin{tabular}{l c c c c c c c c c c}
    \hline
    Datasets & Method & $k$-Means & G-Means & WG-Means & DP-Means & EWDP & RCC & MS & BMS & WBMS\\
    \hline
    GLIOMA        & NMI & $0.499$ & $0.522$ & $0.517$& $0.576$ & $0.675$  & $0.113$ & $0.580$ & $0.546$ & $\mathbf{0.706}$\\
        & ARI & $0.328$ & $0.367$ & $0.373$ & $0.416$ & $0.598$ & $0.004$ & $0.429$ & $0.398$ & $\mathbf{0.618}$\\
    \hline
    Appendicitis  & NMI & $0.157$ & $0.165$ & $0.185$ & $0.158$ &  $0.189$ & $0.193$ & $0.008$ & $0.195$ & $\mathbf{0.249}$\\
   & ARI & $0.230$ & $0.169$ & $0.188$  &  $0.231$& $0.188$ & $0.000$ & $0.014$ & $0.223$ & $\mathbf{0.434}$\\
    \hline     
    Zoo           & NMI & $0.741$ &  $0.801$ & $0.749$ & $0.611$ & $0.859$  & $0.557$ & $0.706$ & $0.841$ &  $\mathbf{0.925}$\\
       & ARI & $0.459$ & $0.646$& $0.559$ & $0.452$ & $0.872$ & $0.039$  & $0.436$ & $0.867$ & $\mathbf{0.953}$\\
    \hline
    Mammogra-  & NMI & $0.231$ & $0.181$ & $0.240$ & $0.191$ & $0.273$ & $0.181$ & $0.181$ & $0.008$& $\mathbf{0.348}$\\
      phic & ARI & $0.292$ & $0.209$ & $0.293$ & $0.257$ & $0.247$ & $0.001$& $0.150$ & $0.001$ & $\mathbf{0.351}$\\
    \hline
    Yale          & NMI & $0.539$ & $0.417$ & $0.521$ & $0.324$ & $0.334$ & $0.268$ & $0.548$ & $0.222$ & $\mathbf{0.693}$\\
                  & ARI & $0.271$ & $0.173$ & $0.227$ & $0.040$ & $0.061$ & $0.031$ & $0.069$& $0.027$ & $\mathbf{0.485}$\\
    \hline
    nci9          & NMI & $0.458$ & $0.394$ & $0.394$ & $0.181$ & $0.211$ & $0.143$ & $0.346$ & $0.394$ & $\mathbf{0.686}$ \\
    & ARI & $0.187$ & $0.100$ & $0.100$ & $0.005$ & $0.018$ & $0.005$ & $0.011$ & $0.086$ & $\mathbf{0.419}$ \\
    \hline
    Lymphoma & NMI & $0.441$ & $0.592$ & $0.690$ & $0.518$ & $0.648$ & $0.243$ & $0.241$ & $0.595$ & $\mathbf{0.778}$\\
    & ARI & $0.269$ & $0.340$ & $0.463$ & $0.088$ & $0.431$ & $0.001$& $0.000$ & $0.458$ & $\mathbf{0.604}$\\
    \hline
    Movement  & NMI & $0.591$ & $0.231$ & $0.328$ & $0.333$ & $0.465$ & $0.639$ & $0.245$ & $0.503$ & $\mathbf{0.663}$\\
    Libras & ARI & $0.309$ & $0.068$ & $0.123$ & $0.113$ & $0.217$ & $0.014$ & $0.090$ & $0.185$ & $\mathbf{0.532}$\\
    \hline
        GCM  & NMI & $0.532$ & $0.484$ & $0.497$ & $0.025$ & $0.497$ & $0.637$ & $0.456$ & $0.649$ & $\mathbf{0.833}$\\
                 & ARI & $0.288$ & $0.248$ & $0.266$ & $0.002$ & $0.266$ & $0.361$ & $0.413$ & $0.536$ & $\mathbf{0.714}$\\
    \hline
    \end{tabular}
    }%
     \caption{Performance Analysis on Real Life Datasets in terms of NMI \& ARI values.}
    \label{real}
\end{table*}
To further demonstrate the efficacy of our proposal, we compare WBMS with the peer algorithms on eight benchmark real datasets. The datasets are taken from the UCI machine learning repository\footnote{\url{https://archive.ics.uci.edu/ml/index.php}} \citep{Dua:2019}, Keel repository\footnote{\url{https://sci2s.ugr.es/keel/datasets.php}} \citep{alcala2011keel} 
and ASU feature selection repository\footnote{\url{http://featureselection.asu.edu/}} \citep{li2018feature}. The microarray GCM data is collected from \citep{ramaswamy2001multiclass}. The data dimensions are reported in Table \ref{data_dimension}. We follow the same computational protocols as stated at the beginning of Section~\ref{exp}. The performance of each algorithm, in terms of NMI and ARI, is reported in Table~\ref{real}, which clearly indicate the superior performance of WBMS on each of the nine real data benchmarks. 
\section{Discussions}
Despite decades of advancement, there is no proper solution for finding the number of clusters and feature weights simultaneously in the clustering problem. For high dimensional datasets, which contain a significant number of non-informative features, most of the existing automated clustering methods fail to identify the discriminating features as well as the number of clusters, resulting in poor performance. To circumvent such difficulties, in this paper, we put forth a novel clustering algorithm known as the Weighted Blurring Mean Shift (WBMS), which not only provides an efficient feature representation of the data but also detects the number of clusters simultaneously. WBMS comes with closed-form updates and rigorous convergence guarantees. We have also mathematically proved that under the nominal assumption of normality of the clusters, WBMS has at least a cubic convergence rate. Through detailed experimental analysis on simulated and real data, WBMS is particularly shown to be useful for high-dimesnional data with many clusters.

\appendix
\section{Optimal Parameter values for Real-Life Datasets}
In this section, we provide the optimal parameter values for real-life datasets. The experimental results for WBMS that are provided in Table \ref{real} are based on the optimal parameter values provided in Table \ref{parameter}.
\begin{table}[h]
    \centering
    \begin{tabular}{l c c}
    \hline
       Datasets & $h$ & $\lambda$\\
       \hline
       GLIOMA & $0.5$ & $1$ \\
       Appendicitis & $1$ & $10$ \\
       Zoo & $0.8$ & $20$\\
       Mammographic & $0.8$ & $10$\\
       Yale & $0.1$ & $10$\\
       nci9 & $0.1$ & $5$ \\
       Lymphoma & $0.5$ & $5$\\
       Movement Libras & $0.1$ & $10$\\
    \hline
    \end{tabular}
    \caption{Optimal Parameter Values for Real-Life Datasets}
    \label{parameter}
\end{table}

\section{Ablation Studies}
\subsection{Abalation Study on $h$}
In this section, we conduct an ablation study to assess the performance of the proposed WBMS algorithm for different values of $h$. We take $h=0.1, 0.5, 0.8$ and $1$ for our experimental study. The algorithm is run according to the procedure described in Section $4$. As a cluster validity index, we use the ARI values between the ground truth and the obtained partition. We choose $\lambda$ corresponding to the highest ARI value. Table \ref{ablation_h} gives the ARI values for all the real-life datasets used in our experiments for different values of $h$. 

\begin{table}[h]
    \centering
    \begin{tabular}{l c c c c}
    \hline
    Datasets & $h=0.1$ & $h=0.5$ & $h=0.8$ & $h=1$\\
    \hline
      GLIOMA   & $0.587$ & $0.618$ & $0.618$ & $0.524$\\
       Appendicitis  & $0.218$ & $0.324$ & $0.401$ & $0.434$ \\
       Zoo & $0.953$ & $0.912$ & $0.754$ & $0.754$ \\
       Mammographic & $0.212$ & $0.295$ & $0.351$ &  $0.320$\\
       Yale & $0.485$ & $0.414$ & $0.234$ &  $0.001$\\
       nci9 & $0.419$ & $0.419$ & $0.340$ & $0.108$\\
       Lymphoma & $0.589$ & $0.604$ & $0.544$ & $0.330$ \\
       Movement Libras & $0.532$ & $0.414$ & $0.021$ & $0.001$\\
       GCM & $0.714$ & $0.531$ & $0.531$ & $0.562$\\
       \hline
    \end{tabular}
    \caption{Average ARI values for WBMS on real-datasets for different values of $h$}
    \label{ablation_h}
\end{table}

\begin{table}[h]
    \centering
    \begin{tabular}{l c c c c}
    \hline
    Datasets & $\lambda=1$ & $\lambda=5$ & $\lambda=10$ & $\lambda=20$\\
    \hline
      GLIOMA   & $0.618$ & $0.618$ & $0.618$ & $0.510$\\
       Appendicitis  & $0.257$ & $0.352$ & $0.434$ & $0.398$ \\
       Zoo & $0.717$ & $0.717$ & $0.906$ & $0.953$ \\
       Mammographic & $0.198$ & $0.202$ & $0.351$ &  $0.351$\\
       Yale & $0.055$ & $0.414$ & $0.485$ &  $0.423$\\
       nci9 & $0.210$ & $0.419$ & $0.419$ & $0.378$\\
       Lymphoma & $0.167$ & $0.604$ & $0.503$ & $0.503$ \\
       Movement Libras & $0.251$ & $0.376$ & $0.532$ & $0.532$\\
       GCM & $0.510$ & $0.714$ & $0.714$ & $0.714$\\
       \hline
    \end{tabular}
    \caption{Average ARI values for WBMS on real-datasets for different values of $\lambda$}
    \label{ablation_lambda}
\end{table}

\subsection{Abalation Study on $\lambda$}

For each of the optimal $h$, described in the previous section, we assess the performance of WBMS for different values of $\lambda$. Since, $\lambda \in [1,20]$, we take $\lambda=1,5,10$ and $20$. For each different value of $\lambda$, we run the algorithm similar to the procedure described in Section $4$. We choose $h$ corresponding to the highest ARI value. For ties, any value of $h$ between the ties can be taken as optimal. Table \ref{ablation_lambda} summarizes the ARI values for all real-life datasets used in our experiments for different values of $\lambda$.

\begin{figure}[t]
    \centering
    \includegraphics[height=0.4\textwidth,width=0.48\textwidth]{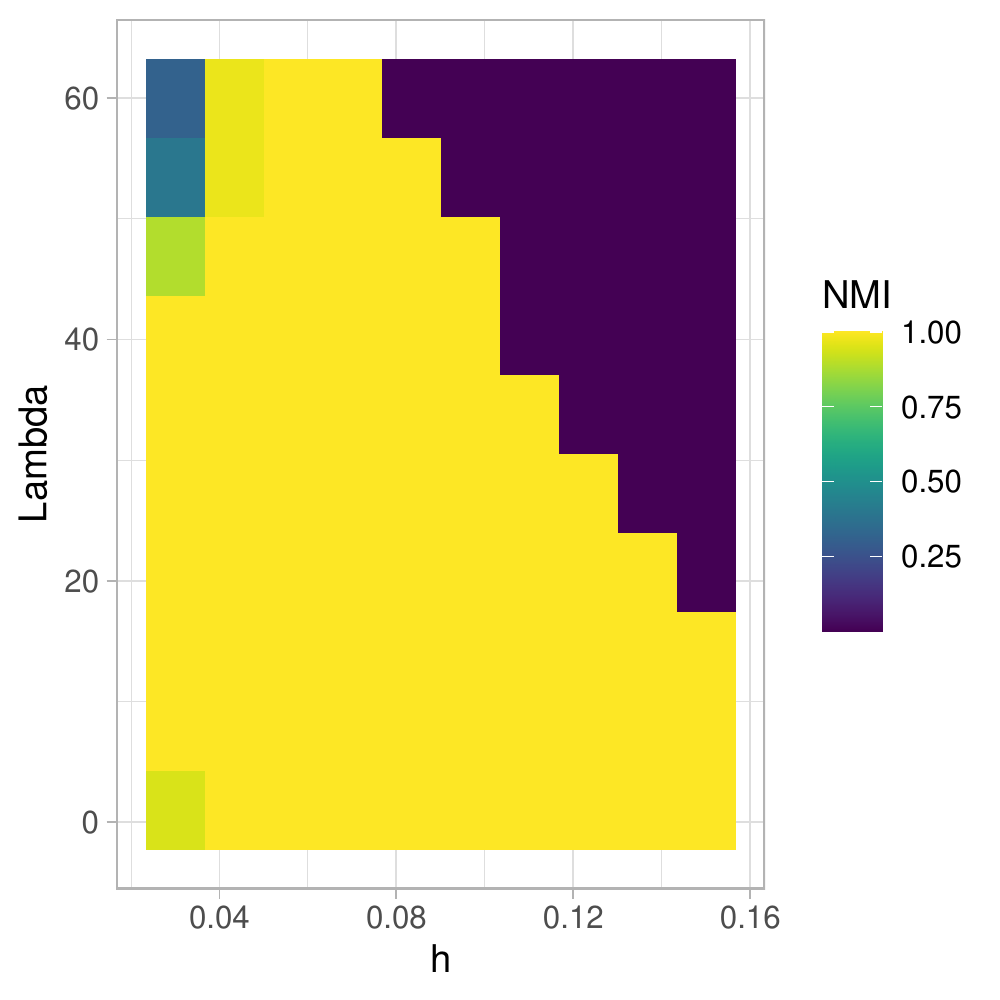}
    \caption{Heatmap of NMI values, plotted against different values of $h$ and $\lambda$, showing that the WBMS is able to recover the perfect clustering (colored in yellow) over a considerable range of the hyper-parameter values.}
    \label{heatmap}
\end{figure}

\begin{figure}[ht]
        \begin{subfigure}[t]{0.45\textwidth}
    \centering
       \includegraphics[height=\textwidth,width=\textwidth]{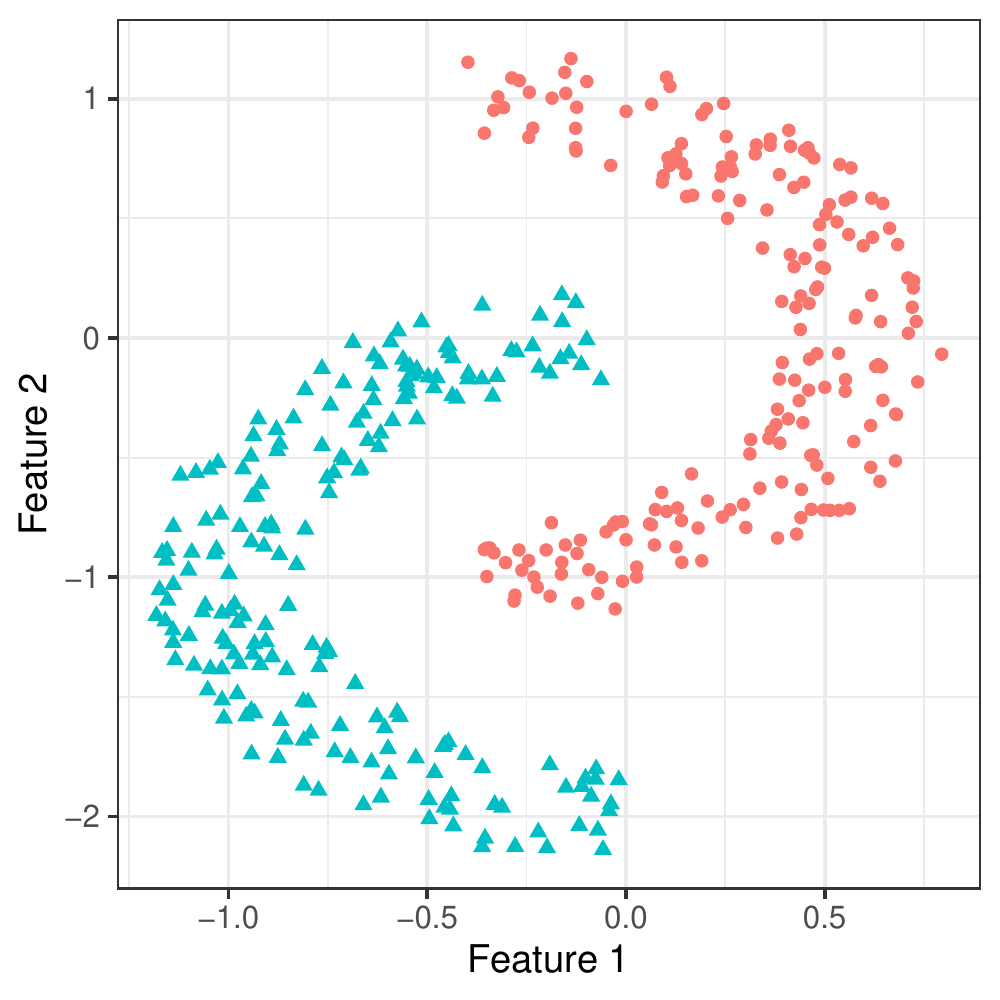}
    \caption{Two-moons data}
        \label{fig:2moon}
        \end{subfigure}
        ~
        \begin{subfigure}[t]{0.45\textwidth}
    \centering
        \includegraphics[height=\textwidth,width=\textwidth]{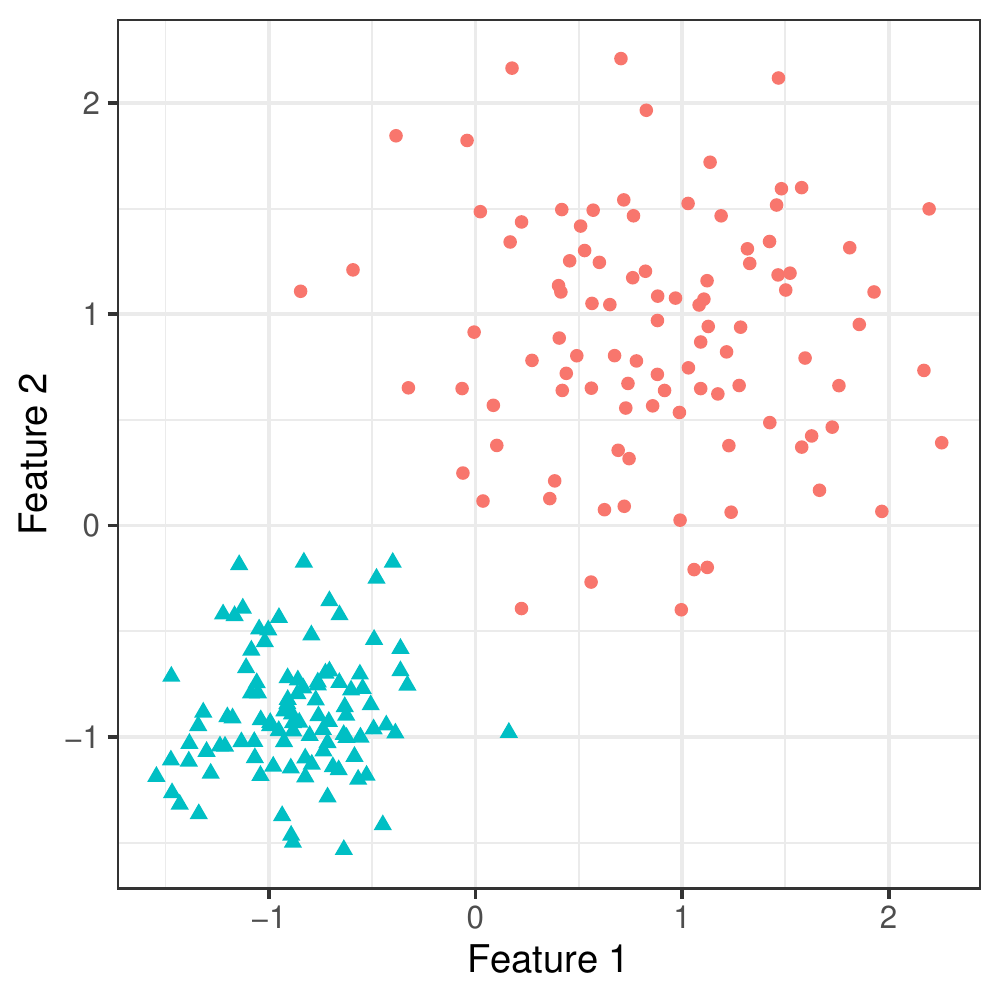}
        \caption{Clusters with varied densities}
        \end{subfigure}
    \caption{Performance of WBMS on some classic synthetic data.}
    \label{eg1}
\end{figure}

\subsection{Sensitivity Analysis}
We also conducted a sensitivity analysis on a toy example with two clusters, each containing 100 points. The first two features of the dataset fully contains the cluster structure of the data, while the rest 30 are simulated independently from a standard normal distribution. For different values of $h$ and $\lambda$, we run the WBMS and compute the NMI values between the ground truth and the obtained partition. In Fig. \ref{heatmap}, we show the heatmap of NMI values, plotted against different values of $h$ and $\lambda$. The hetmap clearly shows that the WBMS is able to recover the perfect clustering (colored in yellow) over a considerable range of the hyper-parameter values.

\section{Additional Experiments}
In this section we show the performance of WBMS on some standard datasets, where mean shift performs perfect clustering, whereas classical centroid-based algorithms fail to do so. To make the problems harder, we appended an additional 30 features, simulated from a standard normal distribution to the data matrix. In Fig. \ref{eg1}, it is easily observed that WBMS performs a perfect clustering on these datasets. 
\section{Run Time Analysis on Real Data}
\begin{table*}[ht]
    \centering
    \caption{Runtime Analysis on Real Life Datasets (Time in Seconds)}
    \label{real}
   \resizebox{\textwidth}{!}{
    \begin{tabular}{l c c c c c c c c c}
    \hline
    Datasets & $k$-Means & G-Means & WG-Means & DP-Means & EWDP & RCC & MS & BMS & WBMS\\
    \hline
    GLIOMA   & $2.19$ & $21.57$ & $30.45$ & $17.48$ & $17.73$  & $23.60$ & $26.96$ & $20.97$ & $20.18$\\
    Appendicitis  & $0.56$ & $3.50$ & $3.42$ & $5.64$ &  $4.96$ & $4.65$ & $2.13$ & $1.36$ & $1.17$\\
                  
    Zoo  & $0.89$ & $2.81$  & $6.06$ & $2.70$ & $3.36$ & $4.29$ & $2.02$ & $1.36$ & $1.29$\\
    Mammographic  & $2.78$ & $5.27$ & $5.78$ & $4.29$ & $5.61$ & $52.69$ & $103.98$ & $81.04$ & $86.21$\\
              
    Yale    & $8.45$ & $16.06$ & $28.63$ & $25.67$ & $27.18$ & $37.50$ & $41.92$ & $189.98$ & $209.18$\\
                  
    nci9     & $9.08$ & $17.25$ & $26.07$ & $29.19$ & $28.96$ & $63.79$ & $68.65$ & $53.48$ & $50.95$ \\
                  
    Lymphoma & $1.31$ & $25.01$  & $34.81$ & $33.42$ & $35.70$ & $64.34$ & $56.47$ & $47.80$ & $42.16$\\
                  
    Movement Libras & $1.98$ & $17.41$ & $21.97$ & $18.09$ & $20.69$ & $43.28$ & $38.27$ & $22.80$ & $24.17$\\
            
    GCM & $12.81$ & $80.67$ & $113.10$ & $153.74$ & $147.32$ & $90.56$ & $308.09$ & $340.61$ & $305.76$\\             
    \hline
    \end{tabular}
    }%
\end{table*}
\end{document}